\documentclass{article}
\usepackage[preprint]{neurips_2019}

\usepackage{microtype}
\usepackage{graphicx}
\usepackage{xcolor}
\usepackage{subfig}
\usepackage{caption}
\usepackage{float}
\usepackage{amsfonts}
\usepackage{algorithm}
\usepackage{algorithmic}
\usepackage{ifdraft}

\usepackage{hyperref}

\usepackage{natbib}

 \usepackage{longtable}
 \usepackage[nomargin,inline,draft]{fixme}
 \fxsetup{theme=color}
 \usepackage{booktabs}
 \usepackage[load-configurations=version-1]{siunitx} 
\usepackage{xspace}
\usepackage{xcolor}
\usepackage{amsmath}
\usepackage{amsthm}
\usepackage{bbm}

\newtheorem{theorem}{Theorem}
\newtheorem{lemma}{Lemma}

\newcommand{\bx}{{\mathbf x}}

\newcommand{\bxi}{{\bx}_{i'}}

\newcommand{\PG}{{\textsc{Greedy}}\xspace}
\newcommand{\ePG}{{\textsc{$\epsilon$-Greedy}}\xspace}
\newcommand{\WS}{\textsc{WeightedSampling}\xspace}
\newcommand{\RNDM}{\textsc{Random}\xspace}

\newcommand{\PA}{{\textsc{$\epsilon$-Greedy}}\xspace}
\newcommand{\PB}{\textsc{WeightedSampling}\xspace}

\newcommand{\Ht}{f^{(t)}}
\newcommand{\pit}{{\bf r}_{i'}^{(t)}}
\newcommand{\pikt}{r_{i'k}^{(t)}} 
\newcommand{\uit}{u_{i'}^{(t)}}

\newcommand{\Suu}{{\cal D''}}

\newcommand{\Ptrain}{{P}}
\newcommand{\Ptest}{{Q}}
\newcommand{\RR}{\mathbb{R}}

\newcommand{\Dtest}{{\cal D'}}

\begin{document}

\title{On weighted uncertainty sampling in active learning}
\author{Vinay Jethava\\
        Seal Software, Kungsgatan 34, Gothenburg 411 19, Sweden\\
        \texttt{vjethava@gmail.com}
}

\maketitle

\begin{abstract}
This note explores probabilistic sampling weighted by uncertainty in active learning. 
This method has been previously used and authors have tangentially remarked on its efficacy. 
The scheme has several benefits: (1) it is computationally cheap, (2) it can be implemented in a single-pass 
streaming fashion which is a benefit when deployed in real-world systems where different subsystems 
perform the suggestion scoring and extraction of user feedback, and (3) it is easily parameterizable. 
In this paper, we show on publicly available datasets that using probabilistic weighting is often beneficial 
and strikes a good compromise between exploration and representation especially when the starting set of labelled points is biased. 
 \end{abstract}

\section{Introduction}
\label{sec:intro}
Traditional active learning focusses on querying domain expert(s) to label examples which are most informative~\citep[see][and references therein]{settles2012active} based on the labels/scores given by the trained classifier. In contrast, a ``representative'' method such as sampling points uniformly at random will find the points that the initial classifier mispredicts with high confidence, albeit, at the cost of large number of annotations. Several methods have explored adaptive sampling based on pre-clustering of the unlabeled data~\citep[see e.g.,][]{dasgupta2008hierarchical}, semi-supervised learning,  etc.  These approaches have the following characteristics: (1) designed with specific classification techniques in mind~(e.g., SVMs), (2) require access to the sample features for clustering, and consequently,  (3) the results depend on the quality of the clustering~\citep{dasgupta2008hierarchical}. 

A closely related problem is identification of {\em unknown unknowns}~\citep{attenberg2011beat}, i.e., identifying test samples where the predictions by a black-box trained model are not representative either due to model bias in training data, data shift between train and test distributions, or some other factor. Recently, \cite{lakkaraju2017identifying} presented an elegant method for this problem using the following approach: (1) greedily partitioning of the test data using an algorithm based on frequent-pattern mining; and, (2) querying random samples from partitions chosen as to optimize the non-stationary utility in a multi-armed bandit setting. \citet{bansal2018coverage} extended this work by defining a coverage-based submodular utility function which allows a greedy algorithm with constant-factor approximation. Both these approach do \emph{not} require access to training data and do not change the classifier (in contrast to traditional active learning methods discussed previously.) The above methods~\citep{lakkaraju2017identifying,bansal2018coverage} use a fixed batch size in their experimental setup. In practice, this one-size-fits-all approach is unsuitable when working with large diverse datasets (e.g., tens of millions of data points). 

We revisit the method of probabilistically sampling weighted by informativity~\citep[see e.g.,][\S 4.3]{angeli2014stanford} that balances ``representative'' and ``informative'' active learning, rather than selecting the most informative samples {\em which can be done in a streaming fashion} (Section~\ref{sec:model}). Section~\ref{sec:experiments} presents experiments on benchmark datasets and we conclude in Section~\ref{sec:conclusions}.

\section{Model}
\label{sec:model}
In  this section, we clarify the assumptions and present a simple result on weighted sampling 
eventually hitting a pocket of ``unknown unknowns''. 

\paragraph{Notation:} We denote feature vector with $\bx \in {\cal X}= \RR^d$ and labels $y\in {\cal Y}$ with ${\cal Y}=\{1, 2, \ldots, k\}$. Let $\Ptrain, \Ptest$ defined on ${\cal X}\times {\cal Y}$ to denote the training and test distributions. We use the notation $p(y)$ to denote the prevalence of class $y$ in the training data; and $q(y)$ to denote the prevalence of class $y$ in the test data.
Let ${\cal D}=\{(\bx_i, y_i)\}_{i=1}^{n}$ and ${\cal D'}=\{(\bx'_j, y'_j)\}_{j=1}^m$ denote the training and test data drawn from distributions $\Ptrain$ and $\Ptest$ respectively. 

We make the following assumptions in this work:
\begin{enumerate}
    \item[A1.] There is ``label shift'' -- the prevalence between training and test distributions has changed but the underlying distribution of features for a class is not different $p(\bx | y) = q(\bx | y),\quad p(y) \ne q(y)$.
\item[A2.] Each class has at least $\chi$ weight in the training and test distributions, i.e, 
$\exists \chi > 0: \; p(y=k) \ge \chi ,\; q(y=k) \ge \chi \quad\forall \; k$.
\end{enumerate}

At iteration $t$, let ${\cal D'}^{(t)}\subseteq {\cal D'}$ denote the set of unlabelled test points and $\Ht$ denote the currently-trained classifier. Let $\pikt := P_{\Ht} (y = k | \bxi )$ denote the probability assigned by classifier $\Ht$ of sample $\bxi$ having label $k$ (i.e., $\sum_k \pikt = 1$).  Let $\uit$ denote the ``informative'' score of sample $\bxi \in {\cal D'}^{(t)}$ assigned by classifier $\Ht$. In this work, we use Shannon's entropy as our informative score  $\uit = H(\pit)= - \sum_{i=1}^K \pikt \log \pikt$.

Let $\Suu \subset \Dtest$ be sample of ``unknown unknowns'' consisting of at least $\beta$ fraction of the data which is being incorrectly labelled by $f^{(0)}$. Considering the case when perfect classification is possible, drawing at least one sample from $\Suu$ will correct this problem. We consider the following sampling scheme which chooses the next sample to query the oracle as \[P(\bxi) \propto   \left( \uit \right)^{d^{(t)}} .\]
where $d^{(t)}$ is a parameter influencing the degree of exploration $(d^{(t)} = 0)$ vs exploitation $(d^{(t)} \to \infty)$ in the $t^{th}$ round of active learning. 

In this work, we only consider $d^{(t)}=1 \,\forall\, t$, which yields Algorithm~\ref{alg:WeightedSampling}~\citep[see e.g.,][]{efraimidis2015weighted}.  This scheme is extremely simple to implement in a streaming fashion, does not require pre-clustering of the data, and will sample from the ``unknown unknowns'' as shown below.

\begin{algorithm}[t]
    \caption{{\sc WeightedSampling}$(\chi, K, f^{(t)}, {\cal D'}^{(t)})$}
\begin{algorithmic}[1]
\FOR{${\bxi} \in {\cal D'}^{(t)}$}
\STATE $\pit = P_{\Ht}(y=\cdot | \bxi) $ \hfill \COMMENT{current classifier $\Ht$}
\STATE ${\bf v}_{i'} = (1-\chi) \pit + \chi {\bf 1}$
\STATE $u_{i'} = H({\bf v}_{i'})$ \hfill \COMMENT{informative score~\eqref{eq:entropy}}
\STATE Choose $a_{i'} \sim {\tt unif}(0, 1)$ and set $k_{i'} = a_{i'}^{1/u_{i'}}$
\ENDFOR
\STATE Return $K$ items with largest keys $k_{i'}$ for labelling
\end{algorithmic} 
\label{alg:WeightedSampling}
\end{algorithm}

\begin{lemma} \label{lem:WeightedSampling} 
The sampling scheme \PB having $N$ rounds each consisting of $K$ samples draws at least one sample from a pocket of ``unknown unknowns'' consisting of $\beta$ fraction of the test data with probability at least $(1-\delta)$, if
\[n_s =  NK \ge \frac{\ln\delta}{\ln (1-p_s)} \]
where $p_s = \beta \left(1 - \frac{\eta - 1 - \ln \eta}{\ln 2 \ln k} \right) $ with $\eta=\frac{- \ln \chi}{1-\chi}$. 
\end{lemma}
One can prove the above result using the following  lower bound on the entropy. 
\begin{theorem}[\citet{cicalese2018bounds}] \label{thm:cicalese} Let ${\bf p} = [p_1, \ldots, p_k]$ be a distribution with $p_1 \ge p_2 \ge \ldots \ge p_k > 0$. If $p_1/p_k \le \rho$, then the entropy $H({\bf p}) = -\sum_{i=1}^{k} p_i \ln p_i$ has the following bound: 
$H({\bf p}) \ge \ln k - \left(\frac{\rho \ln \rho}{\rho - 1} - 1 - \ln \frac{\rho \ln \rho}{\rho - 1} \right) \frac{1}{\ln 2}$.
\end{theorem}
\begin{proof}[Proof of Lemma~\ref{lem:WeightedSampling}]
We note that $\min_{k} \pikt \ge \chi$ (step 3 in Algorithm~\ref{alg:WeightedSampling}) and we can apply Theorem~\ref{thm:cicalese} with $\rho = \frac{1}{\chi}$ to get the following lower bound on \[u_{i'} \ge \zeta \;\forall \;\Dtest\] where   $\zeta =  \ln k - (\eta - 1 - \ln \eta)\frac{1}{\ln 2}$ with $\eta=\frac{\ln (1/\chi)}{1-\chi}$. The total weight on $\Suu$ is at least by $ m \beta \zeta$.  
The total weight on $\Dtest$ is upper bounded by $m \ln k$. Therefore, the probability of picking a sample from $\Suu$ in an independent trial is at least \[p_s \ge \frac{\beta \zeta}{\ln k} = \beta \left(1 - \frac{\eta - 1 - \ln \eta}{\ln 2\ln k}\right).\] 
Therefore, the \PB scheme will draw with probability at least $1-\delta$, one or more samples from $\Suu$ in $\frac{\ln \delta}{\ln (1- p_s)}$ trials. 
\end{proof}

 \section{Experiments}
\label{sec:experiments}
In this section, we compare the following approaches: \textit{RANDOM}, \textit{GREEDY}, \PA ($\epsilon=0.05$), \PB with the same evaluation strategy used in \citet{mussmann2018relationship}.   

We choose the initial set of $N_0=100$ labelled examples have equal prevalence (instead of choosing uniformly at random) of the majority and minority classes. This is a classic approach used to address class imbalance~\citep[see, e.g.][]{chen2004using,lopez2013insight}. We perform $N_b=30$ rounds of active learning with batch size $B=30$ selected based on one of the strategies discussed above. We report accuracy numbers on the hold-out set.   

\begin{table}
\centering
    \begin{tabular}{cc}
 & 
\begin{tabular}{rllll} \hline
Dataset& \RNDM & \PG& \ePG & \WS \\ \hline
oml-823 & $96.4 \pm 0.09$ & $96.7 \pm 0.03$ & $96.69 \pm 0.01$ & $96.49 \pm 0.11$\\
oml-846 & $89.16 \pm 0.17$ & $89.51 \pm 0.08$ & $89.48 \pm 0.13$ & $89.42 \pm 0.15$\\
oml-1169 & $62.05 \pm 0.43$ & $59.53 \pm 1.98$ & $60.76 \pm 0.74$ & \textcolor{cyan}{$61.67 \pm 0.53$}\\
oml-40685 & $96.22 \pm 0.65$ & $97.78 \pm 0.88$ & $98.07 \pm 0.07$ & $98.03 \pm 0.13$\\
oml-1120 & $98.58 \pm 0.23$ & $99.44 \pm 0.03$ & $99.44 \pm 0.03$ & $99.15 \pm 0.11$\\
oml-40668 & $76.97 \pm 0.4$ & $76.81 \pm 0.41$ & $77.35 \pm 0.51$ & $77.34 \pm 0.58$\\
oml-4709 & $88.94 \pm 0.46$ & $89.15 \pm 1.35$ & $89.87 \pm 1.04$ & $89.72 \pm 0.62$\\
oml-1461 & $89.34 \pm 0.2$ & $89.58 \pm 0.32$ & $89.74 \pm 0.22$ & $89.46 \pm 0.26$\\
oml-32 & $98.69 \pm 0.18$ & $99.13 \pm 0.01$ & $99.15 \pm 0.01$ & $99.12 \pm 0.02$\\
oml-351 & $95.07 \pm 0.13$ & $95.21 \pm 0.02$ & $95.22 \pm 0.02$ & $95.19 \pm 0.04$\\
oml-155 & $61.13 \pm 0.4$ & $60.77 \pm 1.4$ & $61.7 \pm 0.49$ & $61.02 \pm 0.32$\\
oml-6 & $96.97 \pm 0.35$ & $98.18 \pm 0.03$ & $98.2 \pm 0.04$ & $97.72 \pm 0.19$\\
oml-151 & $75.9 \pm 0.42$ & $75.78 \pm 1.14$ & $76.32 \pm 0.29$ & $76.05 \pm 0.16$\\
oml-821 & $82.72 \pm 0.47$ & $83.57 \pm 0.31$ & $83.64 \pm 0.3$ & $83.44 \pm 0.25$\\
oml-734 & $87.28 \pm 0.23$ & $87.86 \pm 0.16$ & $87.99 \pm 0.26$ & $87.65 \pm 0.23$\\
oml-1471 & $63.43 \pm 0.9$ & $62.54 \pm 0.88$ & $63.91 \pm 1.33$ & \textcolor{cyan}{$64.92 \pm 1.08$}\\
oml-1596 & $74.2 \pm 0.79$ & $74.61 \pm 1.65$ & $74.82 \pm 1.0$ & $74.59 \pm 0.79$\\
oml-1113 & $99.78 \pm 0.07$ & $99.99 \pm 0.0$ & $99.99 \pm 0.0$ & $99.95 \pm 0.04$\\
oml-1481 & $83.51 \pm 0.83$ & $85.14 \pm 1.31$ & $85.51 \pm 0.46$ & $84.18 \pm 0.48$\\
oml-4534 & $93.23 \pm 0.38$ & $93.89 \pm 0.18$ & $93.92 \pm 0.16$ & $93.71 \pm 0.22$\\
oml-4541 & $59.29 \pm 0.46$ & $56.61 \pm 1.07$ & $57.46 \pm 1.02$ & $58.97 \pm 0.59$\\
covtype & $69.52 \pm 0.19$ & $55.44 \pm 1.35$ & $59.75 \pm 3.15$ & \textcolor{cyan}{\bf $69.39 \pm 0.62$}\\\hline
\end{tabular}
 \\
\end{tabular}
    \caption{Final accuracy (mean and standard deviation under $n_t=5$ independent trials) under badly initialized dataset ($N_0=100$) after $T=30$ active learning rounds each consisting of batch size $B=30$ and different active learning strategies.}
\label{tab:bad-init-accuracy}
\end{table} 

 Table~\ref{tab:bad-init-accuracy} shows the comparison of accuracy results on the hold-out set. We see that \WS achieves the best of both worlds between scenarios where \RNDM dominates and scenarios where \PG, \ePG dominate. This is especially striking in {\tt covtype}, {\tt oml-1169}, {\tt oml-1471}, {\tt oml-155} and {\tt oml-4541} datasets (highlighted in cyan) which encode difficult classification problems (accuracy of resulting classifiers less than 70\%).

 \section{Conclusions}
\label{sec:conclusions}
We observe that weighted sampling (\WS) is a cheap, streaming solution in complex systems to inject randomness in real-world  active learning systems which does not result in significant loss compared to most-uncertain strategy but can be beneficial in setups with badly chosen set of initially labelled instances. Future direction would involve comparison with works targetting unknown unknowns as well as other algorithms targeting exploitation-exploration (e.g., UCB)~\citep[][]{sutton2018introduction}.

\bibliography{readcube_export}

\begin{thebibliography}{12}
\providecommand{\natexlab}[1]{#1}
\providecommand{\url}[1]{\texttt{#1}}
\expandafter\ifx\csname urlstyle\endcsname\relax
  \providecommand{\doi}[1]{doi: #1}\else
  \providecommand{\doi}{doi: \begingroup \urlstyle{rm}\Url}\fi

\bibitem[Angeli et~al.(2014)Angeli, Gupta, Jose, Manning, R{\'e}, Tibshirani,
  Wu, Wu, and Zhang]{angeli2014stanford}
Gabor Angeli, Sonal Gupta, Melvin Jose, Christopher~D Manning, Christopher
  R{\'e}, Julie Tibshirani, Jean~Y Wu, Sen Wu, and Ce~Zhang.
\newblock Stanford’s 2014 slot filling systems.
\newblock \emph{TAC KBP}, 695, 2014.

\bibitem[Attenberg et~al.(2011)Attenberg, Ipeirotis, and
  Provost]{attenberg2011beat}
Josh Attenberg, Panagiotis~G Ipeirotis, and Foster~J Provost.
\newblock Beat the machine: Challenging workers to find the unknown unknowns.
\newblock \emph{Human Computation}, 11\penalty0 (11):\penalty0 2--7, 2011.

\bibitem[Bansal and Weld(2018)]{bansal2018coverage}
Gagan Bansal and Daniel~S Weld.
\newblock A coverage-based utility model for identifying unknown unknowns.
\newblock In \emph{Proc. of AAAI}, 2018.

\bibitem[Chen et~al.(2004)Chen, Liaw, and Breiman]{chen2004using}
Chao Chen, Andy Liaw, and Leo Breiman.
\newblock Using random forest to learn imbalanced data.
\newblock \emph{University of California, Berkeley}, 110:\penalty0 1--12, 2004.

\bibitem[Cicalese et~al.(2018)Cicalese, Gargano, and
  Vaccaro]{cicalese2018bounds}
Ferdinando Cicalese, Luisa Gargano, and Ugo Vaccaro.
\newblock Bounds on the entropy of a function of a random variable and their
  applications.
\newblock \emph{IEEE Transactions on Information Theory}, 64\penalty0
  (4):\penalty0 2220--2230, 2018.

\bibitem[Dasgupta and Hsu(2008)]{dasgupta2008hierarchical}
Sanjoy Dasgupta and Daniel Hsu.
\newblock Hierarchical sampling for active learning.
\newblock In \emph{Proceedings of the 25th international conference on Machine
  learning}, pages 208--215. ACM, 2008.

\bibitem[Efraimidis(2015)]{efraimidis2015weighted}
Pavlos~S Efraimidis.
\newblock Weighted random sampling over data streams.
\newblock In \emph{Algorithms, Probability, Networks, and Games}, pages
  183--195. Springer, 2015.

\bibitem[Lakkaraju et~al.(2017)Lakkaraju, Kamar, Caruana, and
  Horvitz]{lakkaraju2017identifying}
Himabindu Lakkaraju, Ece Kamar, Rich Caruana, and Eric Horvitz.
\newblock Identifying unknown unknowns in the open world: Representations and
  policies for guided exploration.
\newblock In \emph{AAAI}, volume~1, page~2, 2017.

\bibitem[L{\'o}pez et~al.(2013)L{\'o}pez, Fern{\'a}ndez, Garc{\'\i}a, Palade,
  and Herrera]{lopez2013insight}
Victoria L{\'o}pez, Alberto Fern{\'a}ndez, Salvador Garc{\'\i}a, Vasile Palade,
  and Francisco Herrera.
\newblock An insight into classification with imbalanced data: Empirical
  results and current trends on using data intrinsic characteristics.
\newblock \emph{Information Sciences}, 250:\penalty0 113--141, 2013.

\bibitem[Mussmann and Liang(2018)]{mussmann2018relationship}
Stephen Mussmann and Percy Liang.
\newblock On the relationship between data efficiency and error for uncertainty
  sampling.
\newblock volume ICML, 2018.

\bibitem[Settles(2012)]{settles2012active}
Burr Settles.
\newblock Active learning.
\newblock \emph{Synthesis Lectures on Artificial Intelligence and Machine
  Learning}, 6\penalty0 (1):\penalty0 1--114, 2012.

\bibitem[Sutton and Barto(2018)]{sutton2018introduction}
Richard~S Sutton and Andrew~G Barto.
\newblock \emph{Reinforcement Learning: An Introduction}.
\newblock MIT press Cambridge, 2 edition, 2018.
\newblock URL \url{http://incompleteideas.net/book/the-book-2nd.html}.

\end{thebibliography}

\end{document}